\documentclass{article}

\usepackage{amsmath,amssymb}
\usepackage{graphicx}          

 \newcommand{\R}{\mathbb{R}}
 \newcommand{\N}{\mathbb{N}}

  \newcommand{\V}{\mathbb{V}}
 
  \newcommand{\D}{\mathcal{D}}
 
  \newcommand{\dn}{\mathbf{d}}

 \newcommand{\diag}{\mathrm{diag}}

  \newcommand{\zero}{\mathbf{0}}

  \DeclareMathOperator{\interior}{int}
    \DeclareMathOperator{\trace}{trace}
  \DeclareMathOperator{\argmin}{argmin}

\newtheorem{definition}{Definition}
\newtheorem{lemma}{Lemma}
\newtheorem{theorem}{Theorem}
\newtheorem{corollary}{Corollary}
\newtheorem{proposition}{Proposition}

\title{Homogeneous Artificial Neural Network}

\author{Andrey Polyakov\thanks{Univ. Lille, Inria, CNRS, UMR 9189  CRIStAL, Centralle Lille, F-59000 Lille, France, (andrey.polyakov@inria.fr)}}
\date{}

\begin{document}


\maketitle
\begin{abstract}
	The paper proposes an artificial neural network (ANN) being a global approximator for a special class of functions, which are known as \textit{generalized homogeneous}. The homogeneity means   a symmetry of a function with respect to a group of  transformations having topological characterization of a dilation.  In this paper, a class of the so-called linear dilations is considered. A homogeneous universal approximation theorem is proven. Procedures for an upgrade of an existing ANN to a homogeneous one are developed. Theoretical results are supported by  examples from the various domains (computer science, systems theory and automatic control).  
\end{abstract}
\section{Introduction}
\subsection{State of the art}
The universal approximation theorems    \cite{Cybenko1989:MCSS}, \cite{Hornik_etal1989:NN}, \cite{Hornik1991:NN} put limits on what artificial neural networks (ANNs) can theoretically learn. 	These theorems guarantee  the existence of ANN, which approximates  a continuous function on a compact set
with  an arbitrary high precision. A  training of the ANN  is based a compactly supported data as well, while, the trained ANN may be utilized next as a predictor of the function value for an input data, which does not belong to the training set. Sometimes, the new input may be have a rather large distance even from a convex hull of the training set. In the latter case, the ANN is utilized as an extrapolator of the function and the approximation theorems are not applicable.  Moreover, the analysis of the extrapolation error is impossible if  there is no information about the function away from the training set. Therefore, a global extrapolation of a function based on a local data can be provided only under additional assumption about  the class of functions approximated by ANN. This paper deals with approximation of the so-called  generalized homogeneous functions \cite{Zubov1958:IVM}, \cite{Kawski1991:ACDS}, \cite{BhatBernstein2005:MCSS}, \cite{Polyakov2020:Book} and introduces the corresponding homogeneous artificial neural network, which key feature is a global approximation based on local data. 

The generalized homogeneity  is a symmetry of an object (a function, a set, a vector field, etc) with respect to a group of the so-called generalized dilations. Recall that the standard (classical) homogeneity were introduced by Leonhard Euler in 18th century as the symmetry  of a function $x\mapsto f(x)$ 
with respect to the standard (isotropic) dilation of its argument $x\mapsto \lambda x$, namely, $f(\lambda x)=\lambda^{\nu} f(x), \forall \lambda>0, \forall x$, where $\nu\in \R$ is  the homogeneity degree. To the best of author's knowledge, generalized (anisotropic) dilations 
are studied in systems theory since 1950s (see, e.g., 
\cite{Zubov1958:IVM}), where, instead of multiplication of the argument $x$  on a scalar $\lambda>0$, some group of transformations having topological characterization of a dilation is utilized (see, e.g.  \cite{Khomenuk1961:IVM}, \cite{Kawski1991:ACDS}, \cite{Husch1970:Math_Ann}). In particular, if $\dn(s):\V\mapsto \V, s\in \R$ is a one parameter group of coordinate transformations on a normed vector space $\V$ satisfying $\lim_{s\to \pm\infty}\|\dn(s)x\|=e^{\pm \infty}, \forall x\neq \zero$ then $\dn$ is a dilation in $\V$ and a functional $h:\V\mapsto  \R$ satisfying $h(\dn(s)x)=e^{\nu s}h(x)$ is said to be  $\dn$-homogeneous of degree $\nu$.  In this framework, the standard dilation corresponds to the dilation $\dn(s)x=e^sx,s\in \R, x\in \V$.

Generalized homogeneous mathematical objects are studied in various domains of systems theory \cite{Folland1975:AM}, \cite{Hermes1986:SIAM_JCO}, \cite{Kawski1991:ACDS}, \cite{Rosier1992:SCL}, \cite{FischerRuzhansky2016:Book}. Many models of mathematical physics \cite{Polyakov2020:Book} are homogeneous in a generalized sense. Homogeneous models also appear as local approximations of nonlinear systems \cite{Andrieu_etal2008:SIAM_JCO}. The control systems design is frequently based on generalized homogeneity \cite{Levant2003:IJC}, \cite{Perruquetti_etal2008:TAC}, \cite{Polyakov_etal2015:Aut}. This paper deals with functions $\R^n\mapsto \R$, which are  symmetric ($\dn$-homogeneous) with respect to linear dilations $\dn(s)=e^{sG_{\dn}},s\in \R$, where $G_{\dn}\in \R^{n\times n}$ is an anti-Hurwitz matrix\footnote{All eigenvalues have positive real parts}.
The homogeneity implies certain equivalence of local and global information, since any local knowledge about a function can be expanded globally using the dilation symmetry. 

The dilation symmetry (in particular, standard homogeneity) is shown to be useful for analysis and learning of   
ANNs \cite{Neyshabur_etal2015:inBook}. However,  the dilation symmetry of the neural network parameters is usually studied, while the present paper studies a dilation symmetry of input-output relations.  The approximation procedure should take into account the symmetry (if any), otherwise, the geometry of the mapping $g$ may be lost. As shown in the examples given below, this may badly impacts extrapolation capabilities of the ANN. For example, the dilation symmetry (scaling invariance) is important for  the pattern recognition \cite{Bishop2006:Book}, which is expected to be invariant with respect to a zoom (standard dilation) of the image. The systems theory provides many other examples of the generalized homogeneous functions.

\subsection{Basic idea of homogeneous approximation}
Following the classical model of a controlled rotational motion of the rigid body in the 3-D space, the dynamics of the angular velocity of the body is governed  by 
\begin{equation}\label{eq:examp_ODE}
	\dot \omega(t)=J^{-1}(-\omega \times J\omega+\tau),
\quad t\geq 0,
\end{equation}
where 
$\omega(t)=(\omega_1(t),\omega_2(t),\omega_3(t))^{\top}\in \R^n$ is the angular velocity in the body frame, $J=J^{\top}\in \R^{n\times n}$ is the inertia matrix, $\tau(t)\in \R^3$ is the external momentum (e.g., control input) and $\times$ denotes the vector product, i.e., $
\omega \times J\omega=\left(\begin{smallmatrix} 0 & -\omega_3 & \omega_2\\ \omega_3 & 0 & -\omega_1\\ -\omega_2 & \omega_2 & 0\end{smallmatrix}\right)J\omega$.  
The function $h:\R^{6}\mapsto \R^3$ given by  $h(\omega,\tau)=-J^{-1}(\omega \times J\omega+\tau)$ is homogeneous of degree $2$ with respect to a generalized dilation:   $$h(e^{s}\omega,e^{2s}\tau)=e^{2 s}h(\omega,\tau), \quad \forall s\in \R, \forall \omega,\tau \in \R^3.$$ This means that the right-hand side of the differential equation \eqref{eq:examp_ODE} is homogeneous. Considering the problem of data-driven approximation the dynamic equation,  
if $h_{\epsilon}$ is an ANN approximation of $h$ on a compact $\Omega\subset \R^6$ such that $|h(\omega,\tau)-h_{\epsilon}(\omega,\tau)|\leq \epsilon$ for all $(\omega,\tau)\in \Omega$, then $h^{s}_{\epsilon}(\omega,\tau):=e^{2s}h_{\epsilon}(e^{-s}\omega,e^{-2s}\tau)$ is an approximation 
of $h$ on the set $\dn(s)\Omega$, where $\dn(s)=\diag(e^{s}I_3,e^{2s}I_3)$ is the generalized dilation.
Indeed, due to $\dn$-homogeneity of $h$, the approximation error for $h^s_{\epsilon}$ admits the estimate 
$
|h(\omega,\tau)-h^s_{\epsilon}(\omega,\tau)|=e^{2s}|h(e^{-s}\omega,e^{-2s}\tau)-h_{\epsilon}(e^{-s}\omega,e^{-2s}\tau)|\leq e^{2s}\epsilon
$ 
for all $(\omega,\tau)\in \dn(s)\Omega$.
Since $s\in \R^n$ is an arbitrary  number, then the mapping $h^{s}_{\epsilon}$  can be interpreted as as a homogeneous extrapolation of $h_{\epsilon}$ away from $\Omega$. This paper utilizes such an extrapolation  for a global ANN approximator design.
\subsection{Contributions}
The paper has two main contributions to the approximation theory by ANNs:\vspace{-2mm}
\begin{itemize}
	\item  
	A $\dn$-homogeneous ANN with  one hidden level is designed. Such an ANN is shown to be a global approximator of any continuous  $\dn$-homogeneous function on $\R^n$. The corresponding homogeneous approximation theorem is proven. 
	\item A methodology of  an upgrade (transformation) 
	of existing well-trained ANN to a homogeneous one is developed, justified by theoretical proofs and supported by numerical example. 
	 The upgrade assumes that a homogeneous mapping is already approximated by the conventional ANN locally, but the homogeneity parameters (such as the homogeneity degree or the dilation group) are partially known or unknown at all.  Based on available data, these parameters are identified and, next, the existing ANN is homogenized (transformed to a homogeneous ANN) in order to make it a global approximatior of the homogeneous function. This procedure can be interpreted as a \textit{homogeneous extrapolation of ANN}.
\end{itemize}	
\subsection{Structure of the paper}
Preliminaries about generalized homogeneous functions are given in Section 2. Homogeneous ANNs are introduced in Section 3, where, in particular, the homogeneous universal approximation theorem is proven.  Section 4 is devoted to upgrading the classical ANN to a homogeneous one. Numerical examples from computer science, systems theory and automatic control are given in Section 5, where, in particular, the problem of  pattern recognition with scaling invariance  and ANN identification of homogeneous dynamical system \eqref{eq:examp_ODE} are considered.   
\subsection{Notation}
$\R$ is the field of reals; $\R^n_{\zero}=\R^{n}\backslash\{\zero\}$, where $\zero$ is the zero element of a vector space (e.g., $\zero \in \R^n$ means that $\zero$ is the zero vector); $\|\cdot\|$ is a norm in $\R^n$ (to be specified later); 
the matrix norm for $A\in \R^{n\times n}$ is defined as $\|A\|=\sup_{x\neq \zero} \frac{\|Ax\|}{\|x\|}$;  $\lambda_{\min}(P)$ denote a minimal eigenvalue of a symmetric matrix $P=P^{\top}\in \R^{n\times n}$; $P\succ 0$ means that the symmetric matrix $P$ is positive definite; 
$C(\Omega_1,\Omega_2)$ denotes the set of continuous functions  $\Omega_1\subset \R^n\mapsto \Omega_2\subset \R^m$; $C_c(\R^n,\R)$ denotes the set of compactly supported continuous functions; $L^p(\R^n,\R)$ is the Lebesgue space.


\section{Generalized Homogeneity}

\subsection{Linear Dilation}
Let us recall that \textit{a family of  operators} $\dn(s):\R^n\mapsto \R^n$ with $s\in \R$ is  a one-parameter \textit{group} if
$\dn(0)x\!=\!x$, $\dn(s) \dn(t) x\!=\!\dn(s+t)x$, $\forall x\!\in\!\R^n, \forall s,t\!\in\!\R$.
A \textit{group} $\dn$ is 
a) \textit{continuous} if the mapping $s\mapsto \dn(s)x$ is continuous,  $\forall x\!\in\! \R^n$;
b) \textit{linear} if $\dn(s)$ is a linear mapping (i.e., $\dn(s)\in \R^{n\times n}$), $\forall s\in \R$;
c)  a \textit{dilation} \cite{Kawski1991:ACDS} in $\R^n$ if $\liminf\limits_{s\to +\infty}\|\dn(s)x\|=+\infty$ and $\limsup\limits_{s\to -\infty}\|\dn(s)x\|=0$,  $\forall x\neq \zero$.
Any linear continuous group in $\R^n$ admits the representation  \cite{Pazy1983:Book}\vspace{-2mm}
\begin{equation}\label{eq:Gd}
	\dn(s)=e^{sG_{\dn}}=\sum_{j=1}^{\infty}\tfrac{s^jG_{\dn}^j}{j!}, \quad s\in \R,\vspace{-2mm}
\end{equation}
where $G_{\dn}\in \R^{n\times n}$ is a generator of $\dn$. A continuous linear group \eqref{eq:Gd} is a dilation in $\R^n$ if and only if $G_{\dn}$ is an anti-Hurwitz matrix.  In this paper we deal only with continuous linear dilations. 
A \textit{dilation} $\dn$ in $\R^n$ is
i) \textit{monotone} if the function $s\mapsto \|\dn(s)x\|$ is strictly increasing,  $\forall x\neq \zero$;
ii) 	\textit{strictly monotone} if  $\exists \beta\!>\!0$ such that $\|\dn(s)x\|\!\leq\! e^{\beta s}\|x\|$, $\forall s\!\leq\! 0$, $\forall x\in \R^n$.\\
The following result is the straightforward consequence of the existence of the quadratic Lyapunov function for stable linear system \cite{Lyapunov1892:Book}.\vspace{-2mm}
\begin{corollary}\label{cor:monotonicity}
	A linear continuous dilation in $\R^n$ is strictly monotone with respect to the weighted Euclidean norm $\|x\|=\sqrt{x^{\top} Px}$ with $0\prec P\in \R^{n\times n}$ if and only if 
$	PG_{\dn}+G_{\dn}^{\top}P\succ 0,  P\succ 0.$\vspace{-1mm}
\end{corollary}
The standard dilation corresponds to $G_{\dn_1}=I_n$ and $\dn_1(s)=e^{s}I_n$. The generator of the weighted dilation \cite{Zubov1958:IVM} is $G_{\dn_2}\!=\!\diag(r_1,...,r_n)$.
\subsection{Canonical homogeneous norm}

A continuous positive definite function $\|\cdot\|_{\dn}: \R^n\mapsto \R$ satisfying $\|\pm\dn(s)x\|=e^{s}\|x\|_{\dn}, \forall x\in \R^n$ is usually called  a $\dn$-\textit{homogeneous norm} \cite{Kawski1991:ACDS}, \cite{Grune2000:SIAM_JCO}, \cite{Andrieu_etal2008:SIAM_JCO}. 
For the weighted dilation $\dn(s)=\left(\begin{smallmatrix}e^{2s} & 0\\ 0 & e^s\end{smallmatrix}\right)$, a $\dn$-homogeneous norm can be defined, for example, as $\|x\|_{\dn}=\sqrt{|x_1|}+|x_2|$, where $x=(x_1,x_2)^{\top}\in \R^2$. 

The $\dn$-homogeneous norm $\|\cdot\|_{\dn}$ does not satisfy the triangle inequality in the general case,  so, formally, this is not even a semi-norm. However, we would follow the tradition accepted in the systems sciences and call the function $\|\cdot\|_{\dn}$ satisfying the above property by a ``norm'', but, basically, we use this name for the  canonical homogeneous norm (see below) being a norm (in the classical sense) for the vector space $\R^n_{\dn}$ homeomorphic to $\R^n$. The space $\R^n_{\dn}$ has the same set of elements as $\R^n$, but the rules for addition of vectors and multiplication of a vector by a scalar are modified  by means of a special homeomorphism on $\R^n$ (see \cite{Polyakov2020:Book} for more details).
\begin{definition}
	\label{def:hom_norm_Rn}
	Let a linear dilation $\dn$ in $\R^n$ be  continuous and monotone with respect to a norm $\|\cdot\|$.
	A function $\|\cdot\|_{\dn} : \R^n \mapsto [0,+\infty)$ defined as  follows: $\|\zero\|_{\dn}=0$ and \vspace{-2mm}
	\begin{equation}\label{eq:hom_norm_Rn}
		\|x\|_{\dn}\!=\!e^{s_x}, \;  \text{where} \; s_x\in \R: \|\dn(-s_x)x\|\!=\!1, \; x\!\neq\! \zero
		\vspace{-2mm}
	\end{equation}
	is said to be a canonical $\dn$-homogeneous norm \index{canonical homogeneous norm} in  $\R^n$ 
\end{definition}	 
For standard dilation $\dn(s)\!=\!e^{s}I_n$ we have $\|x\|_{\dn}\!=\!\|x\|$. 
In other cases, the canonical homogeneous norm $\|x\|_{\dn}$ with $x\neq \zero$ is implicitly defined by the nonlinear algebraic equation \eqref{eq:hom_norm_Rn}, which always have a unique solution due to monotonicity of the dilation. In some particular cases, this implicit equation has explicit solution. 
\begin{lemma}\cite{Polyakov2020:Book}\label{lem:hom_norm_Lipschitz_Rn}
	If  a linear continuous dilation $\dn$ in $\R^n$  is  monotone with respect to a norm $\|\cdot\|$ 
	then 1)   the function $\|\cdot\|_{\dn} : \R^n\mapsto [0,+\infty)$ implicitly defined by \eqref{eq:hom_norm_Rn} is single-valued and continuous on $\R^n$;
		2)		there exist $\sigma_1,\sigma_2 
		\in \mathcal{K}_{\infty}$:
		\begin{equation}\label{eq:rel_norm_and_hom_norm_Rn}
			\sigma_1(\|x\|_{\dn})\leq \|x\|\leq \sigma_2(\|x\|_{\dn}), \quad \quad 
			\forall x\in \R^n;
		\end{equation}
		3)   $\|\cdot\|$ is locally Lipschitz continuous on $\R^{n}\backslash\{\zero\}$; 4) $\|\cdot\|_{\dn}$ is continuously differentiable on $\R^n\backslash\{\zero\}$ provided that $\|\cdot\|$ is continuously differentiable on $\R^n\backslash\{\zero\}$ and $\dn$ is strictly monotone.
\end{lemma}
For the $\dn$-homogeneous norm $\|x\|_{\dn}$ induced by the weighted Euclidean norm $\|x\|=\sqrt{x^{\top}Px}$  (see Corollary \ref{cor:monotonicity}) we have \cite{Polyakov2020:Book}\vspace{-2mm}
\begin{equation}\label{eq:hom_norm_derivative_Rn}
	\tfrac{\partial \|x\|_{\dn}}{\partial x}=\|x\|_{\dn}\tfrac{x^{\top}\dn^{\top}(-\ln \|x\|_{\dn})P\dn(-\ln \|x\|_{\dn})}{x^{\top}\dn^{\top}(-\ln \|x\|_{\dn})PG_{\dn}\dn(-\ln \|x\|_{\dn})x},\vspace{-2mm}
\end{equation}
\begin{equation}\label{eq:dilation_rates_in_Rn}
	\sigma_1(\rho)=\left\{
	\begin{smallmatrix}
		\rho^{\alpha} & \text{ if }  & \rho\leq 1,\\
		\rho^{\beta} & \text{ if} & \rho >1, 
	\end{smallmatrix}
	\right. \quad \quad \sigma_2(\rho)=\left\{
	\begin{smallmatrix}
		\rho^{\beta} & \text{ if }  & \rho\leq 1,\\
		\rho^{\alpha} & \text{ if} & \rho >1,
	\end{smallmatrix}
	\right.
\end{equation}
where 	$\alpha=\frac{\lambda_{\max}\left(P^{\frac{1}{2}} G_{\dn}P^{-\frac{1}{2}}+P^{-\frac{1}{2}}G^{\top}_{\dn} P^{\frac{1}{2}}\right)}{2}\geq$
$\beta=\frac{\lambda_{\min}\left(P^{\frac{1}{2}} G_{\dn}P^{-\frac{1}{2}}+P^{-\frac{1}{2}}G^{\top}_{\dn} P^{\frac{1}{2}}\right)}{2}>0.$

All $\dn$-homogeneous norms are equivalent, i.e., there exist $0<k_1\leq k_2<+\infty$ such that 
$
k_1\|x\|_{\dn,1}\leq \|x\|_{\dn,2}\leq k_2\|x\|_{\dn,1},  \forall x\in \R^n,
$  
where  $\|\cdot\|_{\dn,1}$ and $\|\cdot\|_{\dn,2}$ are $\dn$-homogeneous norms.

In Section 5 we present an ANN-based approximation of $\|\cdot\|_{\dn}$, which uses the so-called $\dn$-homogeneous projector on the unit sphere
\[
\pi_{\dn}(x)=\dn(-\ln \|x\|_{\dn})x, \quad x\neq \zero.
\]
Obviously, $\|\pi_{\dn}(x)\|_{\dn}=1$ and $\pi_{\dn}(\pi_{\dn}(x))=\pi_{\dn}(x),\forall x\neq \zero$, so $\pi$ is, indeed, a projector on the unit sphere  $\{x\in\R^n : \|x\|_{\dn}=1\}$.
In the case of the canonical homogeneous norm this sphere coincides with $S=\{x\in\R^{n}: \|x\|=1\}$, where $\|\cdot\|$ is the conventional norm in $\R^n$ that induces  $\|\cdot\|_{\dn}$ (see Definition \ref{def:hom_norm_Rn}). 
\subsubsection{Homogeneous functions}
Below we study functions being symmetric with respect to a linear dilation $\dn$. The  dilation symmetry introduced by the following definition is known as the generalized  homogeneity \cite{Zubov1958:IVM}, \cite{Kawski1991:ACDS}, \cite{Rosier1992:SCL}, \cite{BhatBernstein2005:MCSS}, \cite{Polyakov2020:Book}.
\begin{definition}\cite{Kawski1991:ACDS}\label{def:hom_fun}
	Let $\dn$ be a dilation in $\R^n$.
	A function $h:\R^n\mapsto \R$ is  $\dn$-homogeneous of degree $\nu\in \R$ if\vspace{-2mm}
	 \begin{equation}\label{eq:homogeneous_function_Rn}
	 	h(\dn(s)x)=e^{\nu s}h(x), \quad  \quad  \forall s\in\R, \quad \forall x\in \R^n.\vspace{-2mm}
	 \end{equation}
\end{definition}
Local behavior of any homogeneous function $h$ can be expanded globally taking into account the homogeneity degree of the function. The sign of the homogeneity degree also defines some properties of the function \cite{BhatBernstein2005:MCSS}.
\begin{proposition}\label{prop:hom_functional_vs_degree_R^n}
	Let $h :\R^n\mapsto \R$ be a $\dn$-homogeneous function of degree $\nu\!\in\! \R$, $h\!\neq\! \zero$  and $\sup_{x\in S} 
	|h(x)|<+\infty$.
	\begin{itemize}
		\item If $\nu>0$ then  $h$ is bounded on any $\dn$-homogeneous ball $B_{\dn}(r)=\{x\in \R^n: \|x\|_{\dn}\leq r\}$ with $r>0$ and 
		\[\lim_{\|x\|\to 0} h(x)=0\quad \text{ and } \quad \limsup_{\|x\|\to +\infty}|h(x)|= +\infty.
		\]
		\item If $\nu<0$ then $h$ is bounded on any set $\R^n\backslash B_{\dn}(r)$, 
		\[\lim_{\|x\|\to +\infty} h(x)=0\quad \text{ and } \quad \limsup_{\|x\|\to 0} |h(x)|= +\infty.
		\]
		\item If  $\deg_{\dn}(h) =0$ then $h$ is uniformly bounded on 
		$\R^n$; moreover, $h\equiv \mathrm{const}$ if
		$h$ is continuous at $\zero$.
		\item  A non-constant homogeneous function $h$ is continuous at zero if and only  if $\deg_{\dn}(h)>0$.
	\end{itemize}
\end{proposition} 
	Since the identity \eqref{eq:homogeneous_function_Rn}
	with $s=-\ln \|x\|_{\dn}$ yields
	\[
	h(u)=\|x\|^{\nu}_{\dn  } h(\dn(-\ln \|x\|_{\dn})x)
	\]
	then the properties of $h$ claimed above  follow from the inclusion $\dn(-\ln \|x\|_{\dn})x\in S$ and 
	from the properties of the canonical homogeneous norm (see Lemma 
	\ref{lem:hom_norm_Lipschitz_Rn}). 
	
Euler homogeneous function theorem is one of fundamental results characterizing homogeneous mappings.  Traditionally it relates a smooth homogeneous  function and its derivative. The following extension of the Euler homogeneous function theorem also links a locally Lipschitz homogeneous function with its integral and deals with the generalized homogeneity.
Let us denote
\[
R(x,x^*)\!=\!\left\{(y_1,...,y_n)^{\!\top}\!\!\in\! \R^n: \underline{x}_i\!\leq\! y_i\!\leq\!\overline{x}_i, i\!=\!1,...,n\right\}\!,
\]
where  $x\!=\!(x_1,...,x_n)^{\!\top}\!\!\in\! \R^n$,  $x^*\!=\!(x_1^*,...,x_n^{*})^{\!\top}\!\!\in\!\R^n\!$ and $\underline{x}_i=\min\{x_i,x_i^*\}$, $\overline{x}_i=\max\{x_i,x_i^*\}$.
\begin{theorem}[\small Euler homogeneous function theorem]\hfill\newline
	\label{thm:hom_function_theorem_Rn}\index{homogeneous function theorem}
	Let a function $h:\R^n \mapsto \R^n$ be locally Lipschitz continuous on $\R^n\backslash\{\zero\}$ and $h(\zero)=0$. Then  the following three claims are equivalent \vspace{-2mm}
	\begin{itemize}
		\item[1)] $h$ is $\dn$-homogeneous of degree $\nu\in \R$;
		\item[2)] the identity 
		\begin{equation} \label{eq:hom_fun_deriv_relation_Rn}
			\frac{\partial h(x)}{\partial x} G_{\dn}x\stackrel{a.e.}{=}\nu  h(x),
		\end{equation}
		holds almost everywhere on $\R^n$, where $x\in\R^n\backslash\{\zero\}$ and $G_{\dn}\in \R^{n\times n}$ is a generator of the linear dilation $\dn$;
		\item[3)] the  identity
		\begin{equation}\label{eq:hom_function_int_Rn}
			\int^{x}_{x^*}\!\!\!
			(\trace(G_{\dn})+\nu)h(y)-\!\!
			\!\sum\limits_{i: x_i\neq x_i^*}\!\!\!\!\Delta_i(h,y,x,x^*,,G_{\dn}) \;dy
			\!=\!0
		\end{equation}
		holds for $x\!=\!(x_1,...,x_i,...,x_n)^{\!\top}$\!\!, $x^*\!=\!(x_1^*,...,x_i,...,x_n^{*})^{\!\top}$  such that $R(x,x^*)\in \R^n\backslash\{\zero\}$, where $\int^x_{x^*}=\int^{x_1}_{x_1^*}...\int^{x_n}_{x_n^*}$,  $dy=dy_n... dy_1$, $i=1,...,n$, $e_i=(0,...,1,...,0)^{\top}$ and
		\[
	\Delta_i(h,y,x,x^*,G_{\dn})=\frac{
		h(y_{x_i})e_i^{\top}G_{\dn}y_{x_i}-h(y_{x^*_i})e_i^{\top}G_{\dn}y_{x^*_i}}{x_i-x^*_i}
		\]
		with $y_{x_i}\!=\!(y_1,...,x_i,...,y_n)^{\top}$\!, $ y_{x^*_i}\!=\!(y_1,...,x^*_i,...,y_n)^{\top}$\!. 
	\end{itemize}
	Moreover, if $h$ is differentiable  on $\R^n$ (resp., $\R^n\backslash\{\zero\}$), the identity 
	\eqref{eq:hom_fun_deriv_relation_Rn} is defined everywhere on $\R^n$ (resp., $\R^n\backslash\{\zero\}$).
\end{theorem} 
\begin{proof}
		1)$\Rightarrow$2). Let $h\in  \mathcal{H}_{\dn}(\R^n)$ and $\nu=\deg_{\dn}(h)\in \R$. 
	Due to homogeneity, we have  	
	\[
	\begin{split}
		h(x)=&h\left(\dn(\ln \|x\|_{\dn})\dn(-\ln \|x\|_{\dn})x\right)\\
		=&\|x\|^{\nu}_{\dn} h(\dn(-\ln \|x\|_{\dn})x), 
	\end{split}\quad \quad\quad 
	\forall x\neq \zero.
	\]	
	Since, by Rademacher Theorem, the locally Lipschitz continuous function $h$ is differentiable almost everywhere on  $\R^n\backslash\{\zero\}$ then for almost all  $x\in \R^n\backslash\{\zero\}$ using product and chain rules for derivatives we obtain
	\[
	\begin{split}
		\frac{\partial h(x)}{\partial x} \!\stackrel{a.e.}{=}&\!\nu \|x\|_{\dn}^{\nu-1} h(\dn(-\ln \|x\|_{\dn})x) \frac{\partial \|x\|_{\dn}}{\partial x}\\
		&+\|x\|^\nu_{\dn}\frac{\partial h(z)}{\partial z}  \frac{\partial}{\partial x}\!\left( \dn(-\ln \|x\|_{\dn}) x\right),
	\end{split}
	\]
	where $z=\dn(-\ln\|x\|_{\dn})x$.
	Taking into account the identities   $\frac{d}{ds} \dn(s) =G_{\dn} \dn(s)= \dn(s)G_{\dn}, \forall s\in \R$ we conclude
	\[
	\tfrac{\partial \left( \dn(-\ln \|x\|_{\dn}) x\right)}{\partial x}=\dn(-\ln \|x\|_{\dn})- \tfrac{ G_{\dn} \dn(-\ln \|x\|_{\dn})x}{\|x\|_{\dn}}\tfrac{\partial \|x\|_{\dn}}{\partial x}.
	\]
	Finally, since $\frac{\partial  \|x\|_{\dn}}{\partial x} G_{\dn}x=\|x\|_{\dn}$, $\forall x\in 
	\R^n\backslash\{\zero\}$ due to the formula \eqref{eq:hom_norm_derivative_Rn}, then 
	$\frac{\partial \left( \dn(-\ln \|x\|_{\dn})x\right)}{\partial x} G_{\dn}x=\zero$ and the identity 
	\eqref{eq:hom_fun_deriv_relation_Rn} holds.
	
	2)$\Rightarrow$1). For an arbitrary $x\in \R^n\backslash\{\zero\}$ let us consider the function $s\mapsto h(\dn(s)x)$. Since $\dn(s)x\neq \zero$ for all $s\in \R$ then the considered function is locally Lipschitz continuous on $\R$ and for almost all $x\neq \zero$ we have
	\[
	\tfrac{d h(\dn(s)x)}{d s}\stackrel{a.e}{=}\left.\tfrac{\partial h(z)}{\partial z}\right|_{z=\dn(s)x}\tfrac{d\dn(s)x}{ds}\stackrel{a.e}{=}\left.\tfrac{\partial h(z)}{\partial z}G_{\dn}z\right|_{z=\dn(s)x}
	\]
	due to $\frac{d}{ds} \dn(s) =G_{\dn} \dn(s), \forall s\in \R$. Using the identity \eqref{eq:hom_fun_deriv_relation_Rn} we derive
	\[
	\tfrac{d h(\dn(s)x)}{d s}\stackrel{a.e}{=}\nu h(\dn(s)x).
	\]
	Hence, taking into account $h(\dn(0)x)=h(x)$ we derive
	\[
	h(\dn(s)x)=e^{\nu s}h(x), \forall s\in \R, \quad \forall x\neq \zero.
	\]
	Since by assumption $h(\zero)=0$ then $h\in\mathcal{H}_{\dn}(\R^n)$.
	
	2)$\Leftrightarrow$3) Since we have 
	$\frac{\partial h}{\partial x}G_{\dn} x\stackrel{a.e.}{=}\sum_{i=1}^{n}\frac{\partial h}{\partial x_i}e_i^{\top}G_{\dn} x=h(x)$ then, taking into account local Lipschitz continuity of $h$ on $\R^{n}\backslash\{\zero\}$  we conclude that the identity \eqref{eq:hom_fun_deriv_relation_Rn} is equivalent   
	\[
	\sum_{i=1}^{n}\int_{x_1^*}^{x_1}...\int^{x_n}_{x_n^*}\frac{\partial h}{\partial y_i}e_i^{\top}G_{\dn} y \;dy_1...dy_n=\nu \int_{x^*}^{x}h(y)\,dy
	\]
	for all $x,x^*\in \R^n$  such that $R(x,x^*)\in \R^{n}\backslash\{\zero\}$. Changing the integration order
	$$
	\int_{x_1^*}^{x_1}\!\!\!\!\!...\!\int^{x_n}_{x_n^*} \!\!\!\!... dy = \int_{x_1^*}^{x_1}\!\!\!\!\!...\!\int_{x_{i-1}^*}^{x_{i-1}}\!\!\int_{x_{i+1}^*}^{x_{i+1}}\!\!\!\!\!\!\!...\!\int^{x_n}_{x_n^*}\!\!\!\int_{x_i^*}^{x_i} \!\!\!\!... dy_i dy_n dy_{n-1}...
	$$
	and using the integration by parts  
	$$
	\int_{x_i^*}^{x_i}\!\tfrac{\partial h}{\partial y_i}e_i^{\top}\!G_{\dn} y  \;dy_i=
	\left.h(y)e^{\top}_i\!G_{\dn}y\right|^{x_i}_{x_i^*}-\int_{x_i^*}^{x_i}\!\!\!h(y)\tfrac{\partial e_i^{\top}\!G_{\dn} y}{\partial y_i}  \;dy_i
	$$ we derive 
	\[
	\int_{x_1^*}^{x_1}\!\!\!\!\!...\!\int^{x_n}_{x_n^*}\!\!\! \tfrac{\partial h}{\partial y_i}e_i^{\top}G_{\dn} y dy =
	\]
	\[
	\int_{x_1^*}^{x_1}\!\!\!\!\!\!...\!\int_{x_{i\!-\!1}^*}^{x_{i\!-\!1}}\!\!\!\int_{x_{i\!+\!1}^*}^{x_{i\!+\!1}}\!\!\!\!\!\!\!\!...\!\int^{x_n}_{x_n^*}\!\!\left.h(y)e^{\top}_i\!G_{\dn}y\right|^{x_i}_{x_i^*}  dy_n...dy_{i\!+\!1}dy_{i\!-\!1}...dy_1-
	\]
	\[
	\int_{x_1^*}^{x_1}\!\!\!\!\!\!...\!\!\int_{x_{i\!-\!1}^*}^{x_{i\!-\!1}}\!\!\!\int_{x_{i\!+\!1}^*}^{x_{i\!+\!1}}\!\!\!\!\!\!\!\!...\!\!\int^{x_n}_{x_n^*}\!\!\!\!\int_{x_i^*}^{x_i}\!\!\!\!\!h(y)e_i^{\top}\!G_{\dn} e_i \,dy_i  dy_n...dy_{i\!+\!1}dy_{i\!-\!1}...dy_1.
	\]
	Hence, using the identity $\frac{1}{x_i-x_i^*}\int^{x_i}_{x_i^*} 1 dy_i=1$ we derive the formula \eqref{eq:hom_function_int_Rn}. 
\end{proof}

The identity \eqref{eq:hom_function_int_Rn}  is utilized below for identification of the generator $G_{\dn}$ of a dilation group $\dn$ by means of ANNs. 

\section{Homogeneous ANNs}

The statements of universal approximation theorems   \cite{Cybenko1989:MCSS}, \cite{Hornik_etal1989:NN},  \cite{Hornik1991:NN} vary dependently of the restriction to an \textit{activation function}  $\sigma:\R\mapsto \R$ of ANN.  To the best of author's knowledge, the weakest restriction to $\sigma$ is given by the following theorem.
\begin{theorem}[\small Universal approximation theorem, \cite{Leshno_etal1993:NN}]\label{thm:UAT}
	\!\!Let $\sigma\!\in\! C(\R,\R)$ be a  non-polynomial  function and $\Omega\subset \R^n$ be a compact set.
	If $g: \R^n\mapsto \R$ is a continuous function then for any $\epsilon>0$ there exists $N\in \N, A\in \R^{N\times n}, b\in \R^{N}$ and $C\in \R^{1\times N}$  such that
	\begin{equation}\label{eq:ANN_err}
		|g(x)-g_{\epsilon}(x)|\leq\epsilon, \quad \forall x\in \Omega,
	\end{equation}
	\begin{equation}\label{eq:ANN}
		g_{\epsilon}(x)=C\sigma(Ax+b),
	\end{equation}
	where the function $\sigma$ in the above formula is applied in the component-wise manner.
\end{theorem}
The function $g_{\epsilon}$ given by \eqref{eq:ANN} is an  ANN with one hidden layer (also known as Shallow ANN). The above theorem shows that the considered ANN is the universal approximator of any continuous mapping on a compact. For simplicity, \textit{all theoretical results are proven below for scalar-valued functions $\R^n\mapsto \R$. The extension to the vector-valued functions $\R^n\mapsto \R^m$ is straightforward}.

Let us consider, first,  an unknown $\dn$-homogeneous mappings having a known homogeneity degree. We introduce the $\dn$-\textit{homogeneous artificial neural network}  being a global universal approximator of a continuous $\dn$-homogeneous function.


\subsection{Homogeneous universal approximation theorem}
In the case of homogeneous function, the  universal approximation theorem can be  revisited as follows. 
\begin{theorem}\label{thm:UAT_hom_fun}
	Let $\dn$  be a linear continuous dilation in $\R^n$.   Let $\|\cdot\|_{\dn}$ be an arbitrary (explicit or implicit) $\dn$-homogeneous norm in $\R^n$.  Let $\sigma\in C(\R,\R)$ be a non-polynomial function.  
	
	If $h\in C(\R^n\backslash\{\zero\},\R)$ is a $\dn$-homogeneous function of degree $\nu\in \R$  then for any $\epsilon>0$ there exist $N\in \N$, $A\in \R^{N\times n}$, $b\in \R^N$ and $C\in \R^{1\times N}$
	such that  
	\begin{equation}\label{eq:hom_fun_approx_error}
		|h(x)-h_{\varepsilon}(x)|\leq \epsilon \|x\|_{\dn}^{\nu}, \quad \forall x\in  \R^n\backslash\{\zero\},
	\end{equation}
	\begin{equation}\label{eq:hom_fun_ANN}
		h_{\epsilon}(x)=\|x\|_{\dn}^{\nu}C\sigma(A\dn(-\ln \|x\|_{\dn})x+b),  \quad
	\end{equation}
	where the function $\sigma$ in the above formula is applied in the component-wise manner.
\end{theorem}
\begin{proof}
	By definition,  the $\dn$-homogeneous norm $\|\cdot\|_{\dn}$ is a continuous positive definite function satisfying $\|\dn(s)x\|_{\dn}=e^s\|x\|_{\dn}$. Then $\|\cdot\|_{\dn}$ is radially unbounded and 
	the set 
	\[
	S_{\dn}=\{x\in \R^n: \|x\|_{\dn}=1\}
	\]
	is a compact such that $\zero\notin S_{\dn}$. Since
	$\|\dn(-\ln \|x\|_{\dn})x\|_{\dn}=e^{-\ln \|x\|_{\dn}}\|x\|_{\dn}=1$ for any $x\ne 0$ then $\dn(-\ln \|x\|_{\dn})x\!\in\! S_{\dn}$. 
		Let $\|\cdot\|$ be a norm in $\R^n$ and
	\[
	\begin{split}
	 r_1:=&\inf_{x\in \R^n\backslash\{\zero\}}\|\dn(-\ln \|x\|_{\dn})x\|>0,\\
	 r_2:=&\sup_{x\in \R^n\backslash\{\zero\}} 
	 \|\dn(-\ln \|x\|_{\dn})x\|\in [r_1,+\infty).
	 \end{split}
	\]
	Let us consider the set
	\[
	\Omega=\{y\in \R^n : r_1\leq \|y\|\leq r_2\}.
	\]
	By construction, the set $\Omega$ is a compact and  $\zero \notin \Omega$. Since  the function $h$ is continuous on $\Omega$, then,
	by Theorem  \ref{thm:UAT}, for any $\epsilon>0$ there exists $N\in \N$, $A\in \R^{N\times n}$, $b\in \R^N$ and $C\in \R^{1\times N}$ such that 
	\[
	\sup_{y\in \Omega} \left|h(y)-C\sigma(Ay+b)\right|\leq \epsilon.
	\]
	Since, due to  $\dn$-homogeneity of $h$, we have 
	\[
	h(x)=\|x\|^{\nu}_{\dn}h(\underbrace{\dn(-\ln \|x\|_{\dn})x}_{=y\in \Omega}),\quad  \forall x\neq \zero,
	\]
	then the estimate \eqref{eq:hom_fun_approx_error} holds for $h_{\epsilon}$ given by \eqref{eq:hom_fun_ANN}. 
\end{proof}

	\textit{The key feature of the homogeneous ANN is a global approximation based on local measurements}.  Indeed, by the formula \eqref{eq:hom_fun_ANN} the homogeneous ANN  uses the conventional ANN $C\sigma(A\pi_{\dn}+b)$ with the input $\pi_{\dn}=\dn(-\ln \|x\|_{\dn})x$ being compactly supported on the unit sphere $\{\pi\in \R^n: \|\pi\|_{\dn}=1\}$ in $\R^n$.
The classical  universal approximation theorem (see Theorem \ref{thm:UAT}) guarantees the error estimate \eqref{eq:ANN_err} only on this compact. Due to homogeneity, any local estimate of the homogeneous function can be extended globally (see \eqref{eq:hom_fun_approx_error}).  
The approximation error in this case  depends of the norm of the vector $x$ and the homogeneity degree $\nu$ of the function $h$. 
If $\nu>0$ then, due to Proposition \ref{prop:hom_functional_vs_degree_R^n}, $h$ is continuous at $\zero$ and the above error estimate holds at zero as well. For $\nu<0$ the function $h$ is discontinuous and unbounded at $\zero$.  In this case, the approximation error may grow to infinity as $x\to \zero$.
For $\nu=0$ the  error is estimated uniformly on $\R^n$. 


Usually algorithms of ANN learning optimize the parameters $A, b, C$ based on available measurements 
$$y_j=h(x_j), \quad j= 1,..., M,$$
where $M$ is a number of experiments.  Mathematically, the  goal of the ANN learning is to minimize a norm of the residual
\begin{equation}
	e_j=y_j-\|x_j\|_{\dn}^{\nu}C\sigma(A\dn(-\ln \|x_j\|_{\dn})x_j+b).
\end{equation}
It is worth stressing that the homogeneous ANN has structurally the same dependence on the parameters $A,b,C$. So, \textit{all  conventional learning algorithms are applicable for its training}.  
Their study goes out of the scope of this paper since, due to the similarity in structures of ANN and homogeneous ANN, all main advantages and disadvantages of the corresponding methods will be the same in our case.  The aim of this paper  is to demonstrate that the use of dilation symmetry  in ANN design may improve the approximation precision of ANN provided that the mapping, which needs to be approximated, is homogeneous.

\subsection{Homogenization of ANN with known homogeneity parameters}

  By the universal approximation theorem, the ANN is a local approximator of a function, but the homogeneous ANN is a global approximator of a homogeneous  function. The transformation of the  ANN to a homogeneous ANN can be interpreted as  global homogeneous extrapolation of the ANN. In this subsection, we consider the simplest case, when the homogeneity parameters ($G_{\dn}$ and $\nu$) are assumed to be known. In the next section, the problem of homogeneous extrapolation is studied under parametric uncertainty. 

We assume that a conventional ANN \eqref{eq:ANN} is already well trained such that the approximation error \eqref{eq:ANN_err} is somehow minimized.  The linear dilation and the homogeneity degree are assumed to be known. We need to transform the classical ANN \eqref{eq:ANN} to  the homogeneous ANN  \eqref{eq:hom_fun_ANN}.

\begin{proposition}\label{prop:upg_ANN_to_hANN}
	Let  $\Omega\subset \R^n$  be a compact set.
	Let $g_{\epsilon}$ be ANN \eqref{eq:ANN} approximating a  function $h\in C(\R^n\backslash\{\zero\},\R)$  on the compact $\Omega\subset \R^n$ with an error $\epsilon>0$, i.e., the  inequality \eqref{eq:ANN_err} is fulfilled for $g=h$ . Let the function $h$  be $\dn$-homogeneous of degree $\nu\in \R$.
	If there exists a norm $\|\cdot\|$ in $\R^n$ such that 
	\begin{equation}\label{eq:S_in_Omega}
		S=\{x\in \R^n :\|x\|=1\}\subset \interior \Omega
	\end{equation}
	and the dilation $\dn$ be monotone with respect to $\|\cdot\|$, then the homogeneous ANN  \eqref{eq:hom_fun_ANN} with parameters $A,b, C$ taken from the conventional ANN \eqref{eq:ANN} satisfies the error estimate \eqref{eq:hom_fun_approx_error}, where $\|\cdot\|_{\dn}$ is the canonical homogeneous norm induced by $\|\cdot\|$. 
\end{proposition}
\begin{proof}
	By construction the homogeneous ANN $h_{\epsilon}$ coincides with the original ANN $h_{\epsilon}$ on the unit sphere, i.e.,
	\[
	h_{\epsilon}(y)=g_{\epsilon}(y), \quad \forall y \in S.
	\]  
	This means that 
	\[
	|g(y)-g_{\epsilon}(y)|=|h(y)-h_{\epsilon}(y)|\leq \epsilon, \quad \forall y\in S.
	\]
	Since $\|\cdot\|_{\dn}$ is a canonical homogeneous norm induced by the norm $\|\cdot\|$ then  $y=\dn(-\ln \|x\|_{\dn})x\in S$ for all $x\in \R^n\backslash\{\zero\}$ and using the homogeneity we derive
	\[
	\tfrac{|h(x)-h_{\epsilon}(x)|}{\|x\|_{\dn}^{-\nu}}\leq |h(\dn(-\ln \|x\|_{\dn})x)-h_{\epsilon}(-\ln \|x\|_{\dn})x)|\leq \epsilon.
	\]
	The case of homogeneous vector field can be analyzed similarly. 
\end{proof}

The above  proposition shows that the conventional ANN can be straightforwardly upgraded to the homogeneous ANN provided that the information about the dilation symmetry of the function $h$ is available (e.g., from physics).  The idea of the transformation/upgrade is very simple: we just assign a unit sphere $S\subset \Omega$ and expand the values of the $g_{\epsilon}$ from the unit sphere to the whole space $\R^n$ by means of the dilation.  Notice that with such a homogeneous extrapolation the local error estimate \eqref{eq:ANN_err} (being uniform on $\Omega$) is converted to the global $\|x\|_{\dn}$-dependent estimate \eqref{eq:hom_fun_approx_error} on $\R^n$. The similar upgrade can be proposed for a function $h:\D\subset \R^n\mapsto \R$ defined on a $\dn$-homogeneous cone\footnote{A set $\D$ is said to be a $\dn$-homogeneous cone if $\dn(s)\D\subset \D$ for all $s\in \R$.} $\D$. The condition \eqref{eq:S_in_Omega} becomes $S\cap \D \subset \Omega\cap \D$ in this case.

\section{Data-Driven Dilation Symmetry}
In this section we continue the study of  a possibility of transformation of an existing artificial neural network to a homogeneous one.  We assume that an ANN is  given by $g_{\epsilon}: \R^n \mapsto \R^m$ such that
\[
g_{\epsilon}(x)=C\sigma(Ax+b)
\]
approximates a  function $g\in C(\R^{n}\backslash\{\zero\}, \R)$ on a compact $\Omega\subset \R^n\backslash\{\zero\}$, where $A\in \R^{N\times n},b\in \R^{N}, C\in \R^{1\times N}$. The index $\epsilon>0$ indicates that the approximation precision of ANN (see Theorem \ref{thm:UAT})  is given by 
\[
\sup_{x\in \Omega}|g(x)-g_{\epsilon}(x)|\leq \epsilon, \quad \forall x\in \Omega.
\]
The ANN is supposed to be well trained to approximate a homogeneous  function $g=h$ on a compact $\Omega$, i.e., $\epsilon>0$ is assumed to be small enough. But now we consider the case when \textit{some (or any) information about dilation symmetry of the function  $h$ is unavailable}, so, first, the parameters of the dilation group and/or the homogeneity degree have to be identified using  $g_{\epsilon}$, next, a homogeneous ANN has to be designed.

\subsection{Practical homogeneity degree}
Let us study the problem of identification of an \textit{unknown} homogeneity degree of the function $h$   under the assumption that the linear dilation $\dn(s)=e^{sG_{\dn}},s\in \R$ is \textit{known}.  

By Definition \ref{def:hom_fun},  the scalar-valued function $h:\R^n\mapsto \R$ is $\dn$-homogeneous  of degree $\nu\in \R$ if
\begin{equation}\label{eq:ch16_hom_iden_fun}
	h(e^{sG_{\dn}}x)=e^{\nu s} h(x), \quad \forall x\in \R^n, \quad \forall s\in \R,
\end{equation}
where $\dn(s)=e^{sG_{\dn}}$ is a linear dilation in $\R^n$. If $g_{\epsilon}$ approximates a $\dn$-homogeneous function $g=h$ on $\Omega$ with a high enough precision then we should have 
\[
g_{\epsilon}(\dn(s)x)\approx e^{\nu s}g_{\epsilon }(x), \quad x\in \Omega, \quad s\in \R: \dn(s)x\in \Omega.
\]

Let a norm $\|\cdot\|$ in $\R^n$ and the set $\Omega$ be such that $S=\{x\in \R^n: \|x\|=1\}\subset \Omega$ and the dilation $\dn$ be monotone with respect to $\|\cdot\|$.  In this case, due to $\dn(-\ln \|x\|_{\dn})x\in S$, the above approximate identity implies
\begin{equation}
	g_{\epsilon}(\dn(-\ln \|x\|_{\dn})x)\approx \|x\|_{\dn}^{-\nu}g_{\epsilon }(x), \quad \forall x\in \Omega,
\end{equation}
where $\|\cdot\|_{\dn}$ is the canonical homogeneous norm induced by the norm $\|\cdot\|$. Hence, for   $x\neq \zero$  we have 
\begin{equation}
	\nu \approx \frac{1}{\ln \|x\|_{\dn}}\ln \frac{g_{\epsilon}(x)}{g_{\epsilon}(\dn(-\ln \|x\|_{\dn})x)}
\end{equation}
provided that 
\begin{equation}\label{eq:hom_detect}
	0<\frac{g_{\epsilon}(x)}{g_{\epsilon}(\dn(-\ln \|x\|_{\dn})x)}<+\infty.
\end{equation}
Notice that the implication  
$$
h(x)\neq 0 \quad \Rightarrow \quad 0<\frac{h(x)}{h(\dn(s)x)}<+\infty,\quad \forall s\in \R
$$
is necessary  for a function $h$ to be homogeneous. Indeed, if $h(x)\neq 0$ and $\frac{h(x)}{h(-\dn(\ln \|x\|_{\dn})x)}<0$ then the identity \eqref{eq:ch16_hom_iden_fun} is impossible and $h$ is surely non-homogeneous.  Therefore,  the condition \eqref{eq:hom_detect} can be utilized for detection of the non-homogeneity of $g_{\epsilon}$. 
To decrease an impact of the approximation error on a decision about non-homogeneity of $g_{\epsilon}$, we restrict the domain in \eqref{eq:hom_detect} to the compact set  
\begin{equation}
	K_{\delta}=\left\{x\in \Omega:\begin{smallmatrix} \delta\leq \|x\|_{\dn}\leq 1, \\
		\min\{|g_{\epsilon}(x)|,|g_{\epsilon}(\pi_{\dn}(x))|
	\}\geq \delta\end{smallmatrix}\right\},
\end{equation}
where $\pi_{\dn}(x)=\dn(-\ln \|x\|_{\dn})x$ is the $\dn$-homogeneous projector of $x\neq \zero$ on the unit sphere $S$ and  $\delta>0$ is small enough. In this case, if 
\begin{equation}\label{eq:hom_neces_fun}
	\frac{g_{\epsilon}(x)}{g_{\epsilon}(\pi_{\dn}(x))}>0, \quad \forall  
	x\in K_{\delta} 
\end{equation}
then the \textit{practical homogeneity degree of the function} $g_{\epsilon}$ can be defined as 
\begin{equation}\label{eq:prac_hom_degree_fun}
	\nu_{\epsilon }=\frac{1}{M_{\delta}} \int_{K_{\delta}} \frac{1}{\ln \|x\|_{\dn}}\ln \frac{g_{\epsilon}(x)}{g_{\epsilon}(\pi_{\dn}(x))} dx
\end{equation}
provided that $\epsilon>0$ is small enough, where
$
M_{\delta}=
\int_{K_{\delta}} dx.
$
If the condition \eqref{eq:hom_neces_fun} is not fulfilled then $g_{\epsilon}$ is not an approximation of homogeneous function or the approximation precision $\epsilon>0$ is not enough to detect $\dn$-homogeneity.  
\begin{proposition}\label{prop:prac_hom_degree_fun}
	Let $g_{\epsilon}$ be ANN \eqref{eq:ANN} approximating a non-zero continuous function $h\in C(\R^n\backslash\{\zero\}, \R)$ on a compact $\Omega\subset \R^n$ with an error $\epsilon>0$, i.e., the  inequality \eqref{eq:ANN_err} is fulfilled for $g=h$. 
	Let the  function $h$  be $\dn$-homogeneous of degree $\nu\in \R$.
	If there exists a norm $\|\cdot\|$ in $\R^n$ such that 
	\begin{equation}
		S=\{x\in \R^n :\|x\|=1\}\subset \interior \Omega
	\end{equation}
	and the dilation $\dn$ is monotone with respect to $\|\cdot\|$, then for sufficiently small $\delta>0$ it holds
	\begin{equation}
		\nu_{\epsilon}\to \nu \quad \text{ as } \quad \epsilon\to 0
	\end{equation}
	where $\nu_{\epsilon}\in \R$ is given  by \eqref{eq:prac_hom_degree_fun}.
\end{proposition}

\begin{proof} 
	On the one hand,  since $h$ is $\dn$-homogeneous, then, for a sufficiently small $\delta>0$, the compact set $K_{\delta}\subset \Omega\backslash\{\zero\}$ is non-empty. Since $h$ is $\dn$-homogeneous  then $\nu=\frac{1}{\|x\|_{\dn}}\ln \frac{h(x)}{h(\dn(-\ln \|x\|_{\dn})x)}$ for all $x\in K_{\delta}$ and
	\[
	\nu=\frac{1}{M_{\delta}} \int_{K_{\delta}} \frac{1}{\ln \|x\|_{\dn}}\ln \frac{h(x)}{h(\pi_{\dn}(x))} dx.
	\]
	On the other hand, by definition $\nu_{\epsilon}$, we have
	\[
	\nu_{\epsilon }=\frac{1}{M_{\delta}} \int_{K_{\delta}} \frac{1}{\ln \|x\|_{\dn}}\ln \frac{h(x)+o(x,\epsilon )}{h(\pi_{\dn}(x)))+o(\pi_{\dn}(x)),\epsilon)} dx, 
	\]
	where $\sup_{y\in \Omega} |o(y,\epsilon)|\leq \epsilon$ (see \eqref{eq:ANN_err}).
	Since $h$ is assumed to be nonzero  and $\pi_{\dn}(x)=\dn(-\ln \|x\|_{\dn})x\in S$ then 
	$|h(\pi_{\dn}(x))|\geq \underline{h}=\inf_{y\in S\cap K_{\delta}} h(y)>0$. This implies that 
	\begin{equation}\label{eq:ch16_limit_tmp}
		\tfrac{1}{\|x\|_{\dn}}\ln \tfrac{g_{\epsilon}(x)}{g_{\epsilon}(\pi_{\dn}(x))} 
		\to \tfrac{1}{\|x\|_{\dn}}\ln \tfrac{h(x)}{h(\pi_{\dn}(x))}  \quad \text{ as } \quad \epsilon\to 0
	\end{equation}
	uniformly on any compact from $\Omega\cap K_{\delta}\backslash\{\zero\}$, so $\nu_{\epsilon}\to \nu$ as $\epsilon\to 0$.
\end{proof}

The practical homogeneity degree is still well-defined by the formula \eqref{eq:prac_hom_degree_fun} provided that the practical ``necessary'' condition of homogeneity \eqref{eq:hom_neces_fun} is fulfilled. The formula \eqref{eq:prac_hom_degree_fun}  defines the practical homogeneity degree as a mean value of some function. This value can also be approximated as follows
\begin{equation}\label{eq:prac_hom_degree_sum}
	\nu_{\epsilon}\approx\frac{1}{M}\sum_{i=1}^M \frac{1}{\ln \|x_i\|_{\dn}}\ln \frac{g_{\epsilon}(x_i)}{g_{\epsilon}(\pi_{\dn}(x_i))},
\end{equation} 
where the points $x_i\in K_{\delta}\cap \Omega$ are assumed to be uniformly
distributed in $K_{\delta}\cap \Omega$ and $M\in \N$ is a  large enough number. 
\subsection{Practical linear dilation}
Let us study the  identification problem of a dilation symmetry  (if any) for  the ANN $g_{\epsilon}$ which approximates a continuous function $g$.  
It is easy to see that if a function $g$  is $\dn$-homogeneous of degree 
$\nu$ (i.e., $g=h$) then, for any $\gamma>0$ this function is $\tilde \dn$-homogeneous of the degree $\tilde \nu=\gamma \nu$ with respect to the dilation
$
\tilde \dn(s)=\dn(\gamma s), s\in \R.
$
So, \textit{the identification of the dilation symmetry of $$g_{\epsilon}$$ can be reduced to identification of the generator $G_{\dn}\in \R^{n\times n}$ only}, 
while the homogeneity degree can always be  assigned to one of the following constants $\{-1,0,1\}$. In other words, if the sign of the homogeneity degree is unknown, then the three identification problems with $\nu=-1$, $\nu=0$ and $\nu=1$
has to be solved in order to identify a dilation symmetry of $g_{\epsilon}$.  \textit{Without loss of generality we assume below that the homogeneity degree $\nu$ of $g_{\epsilon}$ is known} or fixed to one of the following numbers $\{-1,0,1\}$. We just need to identify  an \textit{unknown}  generator $G_{\dn}\in \R^{n\times n}$ of the dilation $\dn$.


By Theorem \ref{thm:hom_function_theorem_Rn}, 
for any locally Lipschitz continuous (on $\R^n\backslash\{\zero\}$) $\dn$-homogeneous function $h$,
we have 
\begin{equation}\label{eq:ch16_hom_function_int_Rn}
	\int^{x}_{x^*}\!
	(\trace(G_{\dn})+\nu)h(y)-
	\!\sum\limits_{i: x_i\neq x_i^*}\!\!\!\Delta_i(h,y,x,x^*,,G_{\dn})\;dy
	\!=\!0
\end{equation}
for all  $x,x^*\in\!\R^n$ such that $R(x,x^*)\in \R^n\backslash\{\zero\}$.

If $g_{\epsilon}$ approximates a $\dn$-homogeneous function $h$ of degree $\nu$ with a rather high precision $\epsilon$ on a compact $\Omega$ then the identity \eqref{eq:ch16_hom_function_int_Rn} implies that 
\begin{equation}
	\int^{x}_{x^*}\!\!
	(\trace(G)+\nu)g_{\epsilon}(y)-\!
	\!\sum\limits_{i: x_i\neq x_i^*}\!\!\!	\Delta_i(g_{\epsilon},y,x,x^*,G_{\dn})\;dy
	\approx0.
\end{equation}
The only unknown  parameter $G_{\dn}\in \R^{\times n}$ is involved in above approximate identity in the affine  manner. Its identification can be based on a quadratic  optimization. 

Let $\Omega$ have a non-empty interior. For $\delta\in (0,1)$  let
$
\Omega_{\delta}\subset \interior \Omega
$
be such that 
\begin{equation}\label{eq:Omega_delta}
\Omega_{\delta}\dot+B_{\delta}\subset \Omega \quad\text{ and }\quad \Omega_{\delta}\cap B_{\delta}=\emptyset,
\end{equation}
where $B_{\delta}=\{x=(x_1,...,x_n)^{\top}\in \R^n: |x_i|\leq \delta,i=1,...,n\}$. Homogeneous functions may have singularity at zero, so the condition $\Omega_{\delta}\cap B_{\delta}=\emptyset$ exclude this potential singularity point.    
Let us introduce 
\begin{equation}\label{eq:practical_dilation_Rn}
	G_{\epsilon}\in \argmin_{G\in \Xi} J_{\epsilon}(G),
\end{equation}
where $\Xi\subset \R^{n\times n}$ is a set of admissible $n\times n$ matrices and 
\begin{equation}\label{eq:ch16_Jh}
	J_{\epsilon}(G)\!=\! \frac{1}{M_\delta}	\int_{\Omega_{\delta}} \int_{B_\delta} \mathcal{J}^2_{\epsilon}(x^*+z,x^*,G)\, dz dx^*,
\end{equation}
\[
\mathcal{J}_{\epsilon}(x,x^*\!,G)\!=\!\!\int^{x}_{x^*}\!\!\!
(\trace(G)+\nu)g_{\epsilon}(y)-\!
\!\!\sum\limits_{i: x_i\neq x_i^*}\!\!\!\!	\Delta_i(g_{\epsilon},y,x,x^*,G_{\dn})\;dy,
\]
and 
$
M_{\delta}=
\int_{\Omega_{\delta}} dx \cdot \int_{B_{\delta}} dz.
$
If $G_{\epsilon}$ is anti-Hurwitz then, in the view of the following theorem, it can be interpreted as a generator of 
the \textit{practical dilation in $\R^n$}.

\begin{theorem}\label{thm:prac_Gd_fun} Let $\Xi\subset \R^{n\times n}$ be a class of admissible generators of linear dilations.
Let $g_{\epsilon}$ be an ANN \eqref{eq:ANN} approximating a function $h:\R^n\mapsto \R$ on a compact $\Omega\subset \R^n$ with an error $\epsilon>0$, i.e., the  inequality \eqref{eq:ANN_err} is fulfilled for $g=h$. 
If $h$ is locally Lipschitz continuous  on $\R^n\backslash\{\zero\}$  and $\dn$-homogeneous of degree $\nu\in \R$
with respect to a linear dilation  $\dn(s)=e^{G_{\dn}s}, s\in \R^n$ with $G_{\dn}\in \Xi$,
then
\[
\min_{G\in \Xi} J_{\epsilon}(G)\to 0  \quad \text{as} \quad \epsilon\to 0.
\]
\end{theorem}
\begin{proof}
By Theorem \ref{thm:hom_function_theorem_Rn}, we have 
\[
\int^{x}_{x^*}\!
(\trace(G_{\dn})+\nu)h(y)-
\!\!\sum\limits_{i: x_i\neq x_i^*}\!\!\!\Delta_i(h,y,x,x^*,G_{\dn})\;dy=0
\]
for all $x,x^{*}\in \R^{n}\backslash\{\zero\}$ such that $R(x,x^*)\in  \R^{n}\backslash\{\zero\}$. 
Since $g_{\epsilon}(x)\to h(x)$ as $\epsilon\to 0$ uniformly on compacts from $\Omega$ then
\[
\int^{x}_{x^*}\!
(\trace(G_{\dn})+\nu)g_{\epsilon}(y)-
\!\!\sum\limits_{i: x_i\neq x_i^*}\!\!\!\	\Delta_i(g_{\epsilon},y,x,x^*,G_{\dn}) \;dy\to  0
\]
as $\epsilon\to 0$ for all $x,x^{*}\in \Omega_{\delta}$. This means that  
\[
\min_{G\in \Xi} J_{\epsilon}(G)\to 0  \quad \text{as} \quad \epsilon\to 0.
\]
The proof is complete
\end{proof}

Notice that the above theorem  implies that
\[
G_{\dn}\in \mathcal{P}:=\lim_{\epsilon \to 0}\argmin_{G\in \Xi} J_{\epsilon}(G).
\]
provided that $G_{\dn}\in \Xi$.
However, the set $\mathcal{P}$ is not a singleton in the general case and 
the problem of identification of the generator $G_{\dn}$ is ill-posed even if $\Xi=\Xi_{\rm lin}$, where
\begin{equation}
\Xi_{\rm lin}=\{G\in \R^{n\times n}: G \text{ is anti-Hurwitz}\}.
\end{equation}
Indeed, for example, for any $\gamma>0$, the function $h:\R^2\mapsto \R$ given by 
\begin{equation}\label{eq:ch16_demo_h_fun_1}
h(x)=x^{\top}x, \quad x=(x_1,x_2)^{\top}\in \R^2
\end{equation}
is $\dn_{\gamma}$-homogeneous of degree $2$ with respect to the linear dilation 
$
\dn_{\gamma}(s)=e^s\left(\begin{smallmatrix} \cos(\gamma s) & -\sin(\gamma s)\\ \sin(\gamma s) & \cos(\gamma s)\end{smallmatrix}\right),  s\in \R.
$
So, the linear dilation symmetry of the function $h$ given above   cannot be uniquely identified.

To reduce the uncertainty in the dilation identification,  the cost functional can be modified as follow
\begin{equation}\label{eq:practical_dilation_Rn_reg}
G_{\epsilon}=\argmin_{G\in \Xi} J_{\epsilon}(G)+\xi \trace(G^{\top}G),
\end{equation}
where $\Xi=\R^{n\times n}$ and $\xi>0$ is a small  regularization parameter.  
The second term in the above formula is a strictly convex function (being a norm) in the space of square matrices.
 The quadratic optimization problem becomes strictly convex with respect to $G$, so, for any fixed $\epsilon>0$, the matrix $G_{\epsilon}$ is uniquely defined by \eqref{eq:practical_dilation_Rn_reg}. Another possible way to reduce the uncertainty is to restrict the set of admissible generators $\Xi$, for example, to diagonalizable (or simply diagonal) matrices. However, this introduces additional a-priori assumption about  the function. 

The functional   $J_{\epsilon}$ admits the following approximation 
\begin{equation}\label{eq:J_eps_approx}
J_{\epsilon}(G)\approx \frac{1}{M L}\sum_{k=1}^{M} \sum_{j=1}^L \mathcal{J}^2_{\epsilon}(x_k+z_j,x_i,G)\, dz
\end{equation}
where the points $x_k\in \Omega_{\delta}$ are uniformly distributed in $\Omega_{\delta}$, $z_{j}\in B_{\delta}$ are uniformly distributed in $B_{\delta}$ and $M, L\in \N$. Finally, denoting $f(x,x^*\!,y,G)=(\trace(G)+\nu)g_{\epsilon}(y)-\!
\!\!\sum\limits_{i: x_i\neq x_i^*}\!\!\!\!\Delta_i(g_{\epsilon},y,x,x^*,G)$ we derive the following approximation of  the functional $\mathcal{J}^2_{\epsilon}$ for a sufficiently small $\delta\in (0,1)$:
\begin{equation}\label{ee:cal_J_eps_approx}
 \mathcal{J}_{\epsilon}(x_i+z_j,x_i,G)\approx\prod^{n}_{i=1}e^{\top}_iz_j\sum_{\lambda\in\{0,1\}^{n}} f(x_k+z_j,x_k,x_k+\diag(\lambda)z_j,G).
\end{equation}
The presented discretizations of the functionals $J_{\epsilon}$ and $\mathcal{J}_{\epsilon}$ can be utilized for identification of the generator $G_{\epsilon}$ of the practical dilation, however, in this case,  an anti-Hurwitz solution of the optimization problem   \eqref{eq:practical_dilation_Rn} may exist even if the ANN $g_{\epsilon}$ approximates  a non-homogeneous function. So, the practical necessary condition of homogeneity \eqref{eq:hom_neces_fun}
has be utilized in order to check if the ANN  indeed approximates a $\dn_{\epsilon}$-homogeneous (at least on $\Omega$) function. In the view of Proposition \ref{thm:prac_Gd_fun}, a correctness of this conclusion depends on approximation precision of ANN and increases as $\epsilon\to 0$.   
\section{Examples of Homogeneous ANNs}
\subsection{Example 1: Homogeneous  pattern recognition}
The scaling invariance is desirable feature for a pattern recognition by ANN. There are several approaches
to exhibit the required invariance (see, e.g.,  \cite[page 262]{Bishop2006:Book}). One of them is to build the invariance properties into the structure of a neural network. The present example uses the generalized homogeneity for this purpose. 

For simplicity, we assume that we deal with gray-scale images. Following the conventional approach, we  associated any image with a function $(z_1,z_2)\mapsto \phi(z)$ of two arguments, which value $\phi(z)\in [0,255]$ define a tone of the gray color at the coordinates $z=(z_1,z_2)^{\top}$. The value $\phi(z)=0$ corresponds to the white color, while the value $\phi(z)=255$ means that the point with the coordinate $z$ has the black color. A pattern depicted in the image is assumed to be centered. The origin $z=(0,0)$ of the coordinate frame corresponds to the center of the image.  In computer, the coordinates of the points $z$ as well as the possible values of $\phi(z)$ are represented by integers , i.e., $z\in \mathbb{Z}^2, \phi(z)\in \{0,1,2,...,255\}$. 
We extend the domain of $\phi$ to $\R^2$ and  the range of $\phi$ to the interval $[0,255]$, respectively.
This extension of the function $\phi$ can always be made in a continuous manner.  For the points $z$ out of the image range we assign $\phi(z)=0$. In this case, the size of image does not impact to the definition of $\phi\in C_{c}(\R^2,\R)$.

 Traditionally, the values of $\phi(z)$ are utilized as inputs to ANN. Since the function $\phi$ is not homogeneous even with respect to the standard dilation $z\mapsto e^{s} z$, then the efficiency of a homogeneous ANN is questionable in this case.  To homogenize the problem, we introduce the following transformations
\begin{equation}\label{eq:hom_coordinates}
x_i(\phi)=\int_{\R^2} b_i(z)\phi(z)d z, \quad i=1,2,...,n
\end{equation}
where $b_i:\R^2\mapsto \R$ are standard homogeneous functions of positive degree $r_i>-2$, i.e., $b_i(e^sz)=e^{r_is}b(z),\forall s \in \R, \forall z\in \R^2$ and $n\in \N$. For example, we may take 
$b_i(z)=\frac{z^{p_i}_{1}z_2^{q_i}}{r_i!}$, where $m_i\in \{1,2\}$ and $p_i,q_i\in\N\cup\{0\}$, $r_i=p_i+q_i$. In this case, the functions $b_i$ are elements of polynomial basis in $L^2(\Omega,\R)$, where $\Omega\subset \R^2$ contains a support of $\phi$. So, any function $\phi\in C_c(\R^2,\R)\cap L^{2}(\Omega,\R)$ can be approximated as  linear combination of the polynomial basis functions $b_i$ with the coefficients dependent on  the vector $x=(x_1^{\top},...,x_n^{\top})^{\top}\in \R^{2n}$  related with coordinates of $\phi$ in the polynomial basis. Since the function $\phi$ is compactly supported with the support dependent on the image size, the computation of $x_i$ is not a difficult problem and it can be easily parallelized.

 Notice that  the zoom of the image corresponds to the  coordinate transformation $z\mapsto 
e^{s} z$ with $s\in \R$, which  implies  the 
transformation $\phi\mapsto \phi_s$, where $\phi_s(z):=\phi(e^{-s}z)$. Hence, we derive
\[
\begin{split}
x_i(\phi_s)=&\int_{\R^2} b_i(z)\phi(e^{-s}z)d z\\
=&e^{2s} \int_{\R^2} e^{i_is}b_i(e^{-s}z)\phi(e^{-s}z)d (e^{-s}z)\\
=&
e^{(2+p_i)s}x_i(\phi)
\end{split}
\]
In other words, the dilation $z\mapsto e^sz$ (the zoom of the image) implies the dilation of the coordinates $x_i\mapsto e^{(2+r_i)s}	x_i$, $i=0,...,n$. The vector $x=(x_1,...,x_n)^{\top}$ can be utilized as an input to an ANN, which is going to be designed  $\dn$-homogeneous of degree $0$ with respect to  the dilation
\[
\dn(s)\!=\!\diag(e^{(2+r_1)s},...,,e^{(2+r_n)s}), \quad  s\!\in\! \R.
\]
In this case, the scaling/zoom of the image  will not impact  the output of the $\dn$-homogeneous ANN. Therefore, once being trained, such ANN can be utilized for the pattern recognition independently of the scaling of the function $\phi\in C_{c}(\R^2,\R)$. However, scaling of the image maps a function $\phi$ having a discrete domain $\mathbb{Z}^2$. In this case a reduction of the image size may lead to  a degradation of the image quality and impossibility of its recognition. So, the $\dn$-homogeneous ANN will have the natural limitation in recognition of  very reduced images.    

For numerical illustration, we consider the 8 images of the characters shown on Figure \ref{fig:char}. Each image has the size  100 $\times$ 100 pixels. The components of the input vector $x$ are defined by the formula \eqref{eq:hom_coordinates} with $n=8$ and $b_1=1,b_2=z_1, b_3=z_2, b_4=0.5z_1 z_2,b_5=0.5x^2,b_6=0.5y^2,b_7=x^2y/6;b_8=xy^2/6$.  To demonstrate the computational robustness of the method, the coordinates $x_i$ are computed using the following approximation of  the integral \eqref{eq:hom_coordinates}:
	\[
	x_i(\phi)\approx \sum_{k=1}^{100}\sum_{j=1}^{100} b_i(k-50,j-50)\phi(k,j).
	\] 
	
\begin{figure}[h!]
	\centering
	\includegraphics[width=10mm,height=10mm]{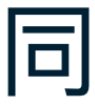}\hspace{3mm}
	\includegraphics[width=10mm,height=10mm]{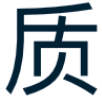}\hspace{3mm}
	\includegraphics[width=10mm,height=10mm]{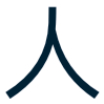}\hspace{3mm}
	\includegraphics[width=10mm,height=10mm]{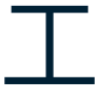}\hspace{3mm}
	\includegraphics[width=10mm,height=10mm]{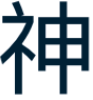}\hspace{3mm}
	\includegraphics[width=10mm,height=10mm]{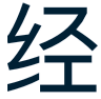}\hspace{3mm}
	\includegraphics[width=10mm,height=10mm]{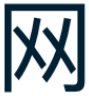}\hspace{3mm}
	\includegraphics[width=10mm,height=10mm]{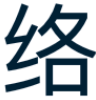}
	\caption{The characters for recognition}\label{fig:char}
	\end{figure}
 
The ANN is assumed to have $k=8$ outputs. For the $i$-th character at the input,  only the $i$-th output must have the value $1$. All other outputs have to be zero in this case.  The training set consist of nominal  images of the size 100 $\times$ 100 pixels and the scaled images the size  75$\times$ 75 pixels, 50 $\times$ 50 pixels  and 35 $\times$ 35 pixels, respectively. The total number of images in the training set is 32.
	 The matrices $A\in \R^{N \times n}$ and $B\in \R^{N\times 1}$ are selected randomly with $N=8$. The matrix $C\in \R^{k \times N}$ is  obtained by means of least square optimization on the training set.
   The trained homogeneous ANN has been tested on  images, which are  scaled (zoomed) to 30\%+$i$10\% with $i=0,1,,...,17$. Notice that the decrease of the size of the nominal image implies a degradation of the image quality  (see Fig \ref{fig:char_zoom}).  ll scaled images have been successfully recognized. 
%
   
   To test the robustness with respect to the noise,   a normally distributed noise has been added to the nominal images. About 50\% of pixels have been corrupted by the noise. The homogeneous ANN successfully recognized the noisy images as well.   The construed homogeneous ANN have good capabilities of pattern recognition with relatively small set of parameters to be tuned (just $8\times 8$ matrix $C$ of real numbers). 

%

\subsection{Example 2: Identification of homogeneous system}
Let us consider now the rotation dynamics \eqref{eq:examp_ODE} of the rigid body   with $J=\left(\begin{smallmatrix} 
	1 &0.3 &0.1\\ 0.3 & 1.2 & 0.2\\
	0.1 & 0.2 & 0.8
	\end{smallmatrix}\right)$. First, we approximate the  right-hand side of  \eqref{eq:examp_ODE}
\[
h(x)=-J^{-1}(\omega \times J\omega+\tau), \quad x=(\omega^{\top},\tau^{\top})^{\top}\in \R^6
\]
by means of the conventional ANN  with the sigmoid activation function
\[
g_{\epsilon}(x)=C\sigma(Ax+b)
\]
where $A\in \R^{N\times 6}, b\in \R^N$ are randomly selected for  $N=500$ and $C\in \R^{3\times N}$ is obtained by means of the least square error minimization for $M=20000$ measurements $y_i=h(x_i)$ of the function $h$ at the points $x_{i}\in \R^{6}$ randomly uniformly distributed in the set 
\[
\Omega=\{x\in \R^6: 0.95\leq |x|\leq 1.05\}.
\]
On $\Omega$ the obtained ANN has the following  approximation precision: 
$$
|g_{\epsilon}(x)-h(x)|\leq \epsilon\approx 0.0110, \quad x\in \Omega
$$
which is approximated by taking $20000$ random points in $\Omega$.
\subsubsection{Homogenization with known $\nu$ and $G_{\dn}$}By Proposition \ref{prop:upg_ANN_to_hANN}, the conventional ANN $g_{\epsilon}$ can be upgraded to a homogeneous ANN  $h_{\epsilon}$ by the formula \eqref{eq:hom_fun_ANN}, where $G_{\dn}=\diag(1,1,1,2,2,2)$, $\nu=2$ and the canonical homogeneous norm $\|\cdot\|_{\dn}$ is induced by the usual Euclidean norm $|\cdot|$. Such an upgrade slightly  improve the approximation precision ($L^{\infty}$-norm of the error) on the set $\Omega$. Moreover, due to homogeneity of $h$, the approximation precision of the homogeneous ANN (hANN), on the domains $\{x\in \R^6:|x|\leq 0.95\}$ and $\{x\in \R^6:|x|\geq 1.05\}$
is much better (in times) than the precision of the classical ANN.  Table \ref{tab:ANN_hANN} presents the comparison of the approximation precision for the domains
\[
\begin{array}{ll}
\Omega_{-4}=\{x\in \R^{6}: 0<|x|\leq 0.25\}, & \Omega_{4}=\{x\in \R^{6}: 1.75<|x|\leq 2\},\\
\Omega_{-3}=\{x\in \R^{6}: 0.25<|x|\leq 0.5\}, & \Omega_{3}=\{x\in \R^{6}: 1.5<|x|\leq 1.75\},\\
\Omega_{-2}=\{x\in \R^{6}: 0.5<|x|\leq 0.75\}, &\Omega_{3}=\{x\in \R^{6}: 1.25<|x|\leq 1.5\},\\
\Omega_{-1}=\{x\in \R^{6}: 0.75<|x|\leq 1\},& \Omega_{1}=\{x\in \R^{6}: 1<|x|\leq 1.25\}.
\end{array}
\]
The table contains the estimates of the norms $\|g_{\epsilon}(x)-h(x)\|_{L^{\infty}(\Omega_k,\R^6)}$ calculated using  20000 randomly selected points in $\Omega_{k}$, $k=-4,...,4$.
\begin{table}
	\centering
	\begin{tabular}{|c|c|c|c|c|}
	\hline 
\textbf{Domain}	& \textbf{ANN} &\textbf{hANN} & \textbf{hANN} & \textbf{hANN}\\
& & $\nu=2$& $\nu=\nu_{\epsilon}$& $\nu=1, G_{\dn}=G_{\epsilon}$\\
\hline 
$\Omega_{-4}$ 	&  0.0864& 0.0010 &0.0134 &0.0284\\
$\Omega_{-3}$ 	&0.0821  & 0.0024 & 0.0168& 0.0286\\
$\Omega_{-2}$ 	& 0.0583 & 0.0038 & 0.0170& 0.0213\\
$\Omega_{-1}$ 	&0.0274  & 0.0063 & 0.0135& 0.0122\\
$\Omega_0=\Omega$ 	&  0.0110& 0.0074 & 0.0082& 0.0094\\
$\Omega_{1}$ 	& 0.0612 & 0.0090& 0.0193& 0.0293\\
$\Omega_{2}$ 	& 0.2257 & 0.0115 & 0.0417&  0.0787\\
$\Omega_{3}$ 	&  0.5056& 0.0153 & 0.0773& 0.1551\\
$\Omega_{4}$ 	&  1.0409& 0.0192 &0.1160& 0.3097\\
\hline
	\end{tabular}   
\caption{Approximation precision of the conventional and homogeneous ANN}\label{tab:ANN_hANN}
\end{table}

\subsubsection{Homogenization using practical homogeneity degree} Let us assume that the homogeneity degree $\nu$ of the function $h$ is unknown, but the dilation $G_{\dn}$  is known. By Proposition \ref{prop:prac_hom_degree_fun}, we can identify a practical homogeneity $\nu_{\epsilon}$ using the conventional ANN,  which is already trained on the set $\Omega$.
We compute $\nu_{\epsilon}$ by the formula \eqref{eq:prac_hom_degree_sum} for randomly selected points $x_i\in \Omega$, where $i=1,2,...,M=2000$. This gives the practical estimate of the homogeneity degree
$\nu_{\epsilon}=1.9647$ being rather close to the real value $\nu=2$. The approximation precision of the hANN with the practical homogeneity degree is given in Table \ref{tab:ANN_hANN}. The use of homogeneous ANN (even with the practical homogeneity degree) essentially improves  the extrapolation capabilities of the ANN.
 
 \subsubsection{Homogenization using practical dilation symmetry} 
 Now we assume that the dilation is unknown and we need to approximate $G_{\dn}$ using the already well-trained ANN $g_{\epsilon}$.  For the data-driven identification of the dilation symmetry, we consider the set $\Omega_{\delta}=\{x\in \R^n: 0.98\leq |x|\leq 1.02\}$ and optimize  the approximate functional 
 \eqref{eq:J_eps_approx} with $M=4000$, $L=1, z_1=\delta(1,...,1)^{\top},\delta=0.01$.  We also restrict the class of admissible generators to diagonal matrices. For $\nu=1$ the obtained matrix $G_{\epsilon}=\diag(0.5199,
 0.5150,
 0.4972,
 1.0021,
 0.9965,
 1.0027)$, which, obviously,  corresponds to $G_{\dn}$ up to the scaling of $G_{\epsilon}$ and $\nu$ by  the multiplier $\gamma=2$ (please see explanations in the beginning of Section 4.2 for more details). 
 The homogeneous ANN constructed in this way is a homogeneous extrapolation  of the ANN $g_{\epsilon}$. In all domain $\Omega_{k}$ except  the training set $\Omega_{0}$, the approximation precision of hANN is better in times comparing with the  original ANN (see Table \ref{tab:ANN_hANN}).

 Notice that  For $\nu=0$ and $\nu=-1$,
 the optimization problem does not provide anti-Hurwitz solutions. This implies  impossibility of non-positive homogeneity degree of the function.

\subsection{Example 3: Homogeneous control implementation}

The canonical homogeneous norm $\|\cdot\|_{\dn}$ (see Definition \ref{def:hom_norm_Rn}) is utilized in  control theory for fast/robust regulation of  linear plants \cite{Polyakov_etal2015:Aut}, \cite{Polyakov2020:Book}:
\[
\dot x=Ax+Bu, \quad u=K_0+K\dn(-\ln \|x\|_{\dn})x,
\]
where $x(t)\in \R^n$ is the system state, $u(t)\in \R^m$ is the control input, $A\in \R^{n\times n}$, $B\in \R^{n\times m}$ are system matrices and $K_0,K\in \R^{m\times n}$ are control gains.  
The methodology for tuning of the control parameters $K_0,K$ of the above nonlinear controller is well developed, but the homogeneous norm $\|\cdot\|_{\dn}$ is defined implicitly (see Definition \eqref{def:hom_norm_Rn}). The practical implementation of this implicit regulator requires of an application of numerical numerical procedure (e.g., bisection method) for computation of $\|x\|_{\dn}$ at each $x\in \R^n$.  For a real-life control application, the norm $\|x(t)\|_{\dn}$ has to be computed in a real time. This may be impossible, if the computational power of the  digital  controller is low. In the view of Theorem \ref{thm:UAT_hom_fun}, the implicitly defined  norm $\|\cdot\|_{\dn}$ can be approximated with arbitrary high precision by means of an explicitly defined norm $\|\cdot\|_{\dn^*}$ and the homogeneous ANN \eqref{eq:hom_fun_ANN}:
\begin{equation}\label{eq:hann_norm}
\|x\|_{\dn}\approx \|x\|_{\dn^*}C\sigma(A\dn(-\ln \|x\|_{\dn^*})x+b).
\end{equation}
For example, let us consider the weighted dilation $\dn(s)=\diag(e^{5s}, e^{4s})$ in $\R^2$ and the canonical homogeneous norm $\|\cdot\|_{\dn,1}$ induced by the weighted  Euclidean norm $\|\cdot\|=\sqrt{x^{\top}Px}$ with 
$P=\left( \begin{smallmatrix} 0.1157 & 0.0194\\0.0194& 0.1186 \end{smallmatrix}\right)$. 
The explicit homogeneous norm  is defined as follows
\[
\|x\|_{\dn^*}=	\left(\Psi^{\top}(x)Q\Psi(x)\right)^{1/10}, \quad Q=\left( \begin{smallmatrix} 0.1225 & 0.0176 \\
0.0176 &0.0720\end{smallmatrix}\right),
\]
where $x=(x_1,x_2)^{\top}\in \R^2$, $\Psi(x)=\left(
\begin{smallmatrix}
 x_1\\
 |x_2|^{5/4}sign(x_2)
\end{smallmatrix}
\right)$ and the matrix $Q$ is selected such that the difference between $\|\cdot\|_{\dn}$ and $\|\cdot\|_{\dn^*}$ is minimized \cite{Polyakov2020:Book}. On the  unit sphere $S=\{x\in \R^2: \|x\|=1\}$ the difference between norms is about $\epsilon\approx 0.007$. Such an error is unacceptable for  some control applications, where a high control precision is required.  To improve the approximation, we use the homogeneous ANN \eqref{eq:hann_norm} with the sigmoid activation function $\sigma$ and randomly selected parameters $A\in \R^{N\times 1}$ and $b\in \R^N$. 
 By definition, the canonical homogeneous norm equals $1$ on the unit sphere, i.e., $\|x\|_{\dn}=1$ for all $x\in S$. 
 The vector $C\in \R^{1\times N}$ is optimized using least square criterion
 \[
 \sum_{i=1}^{M}\left(1\!-\!\|x_i\|_{\dn^*}C\sigma(A\dn(-\ln \|x_i\|_{\dn^*})x+b)\right)^2\to \min_{C\in \R^{1\times N }},
 \]
 where the points $x_i$ are randomly  distributed on the unit  sphere $S$. For $N=10$ such homogeneous ANN $$\|x\|_{\dn,\epsilon}=\|x\|_{\dn^*}C\sigma(A\dn(-\ln \|x\|_{\dn^*})x+b)$$ improves the approximation precision of $\|x\|_{\dn}$ on the unit sphere in two times $\epsilon\approx 0.0034$ in comparison with the norm $\|\cdot\|_{\dn^*}$(see Fig. \ref{fig:hnorm_sigmoid}). 

	\begin{figure}
 		\hspace{40mm}
 		\includegraphics[width=55mm,height=40mm]{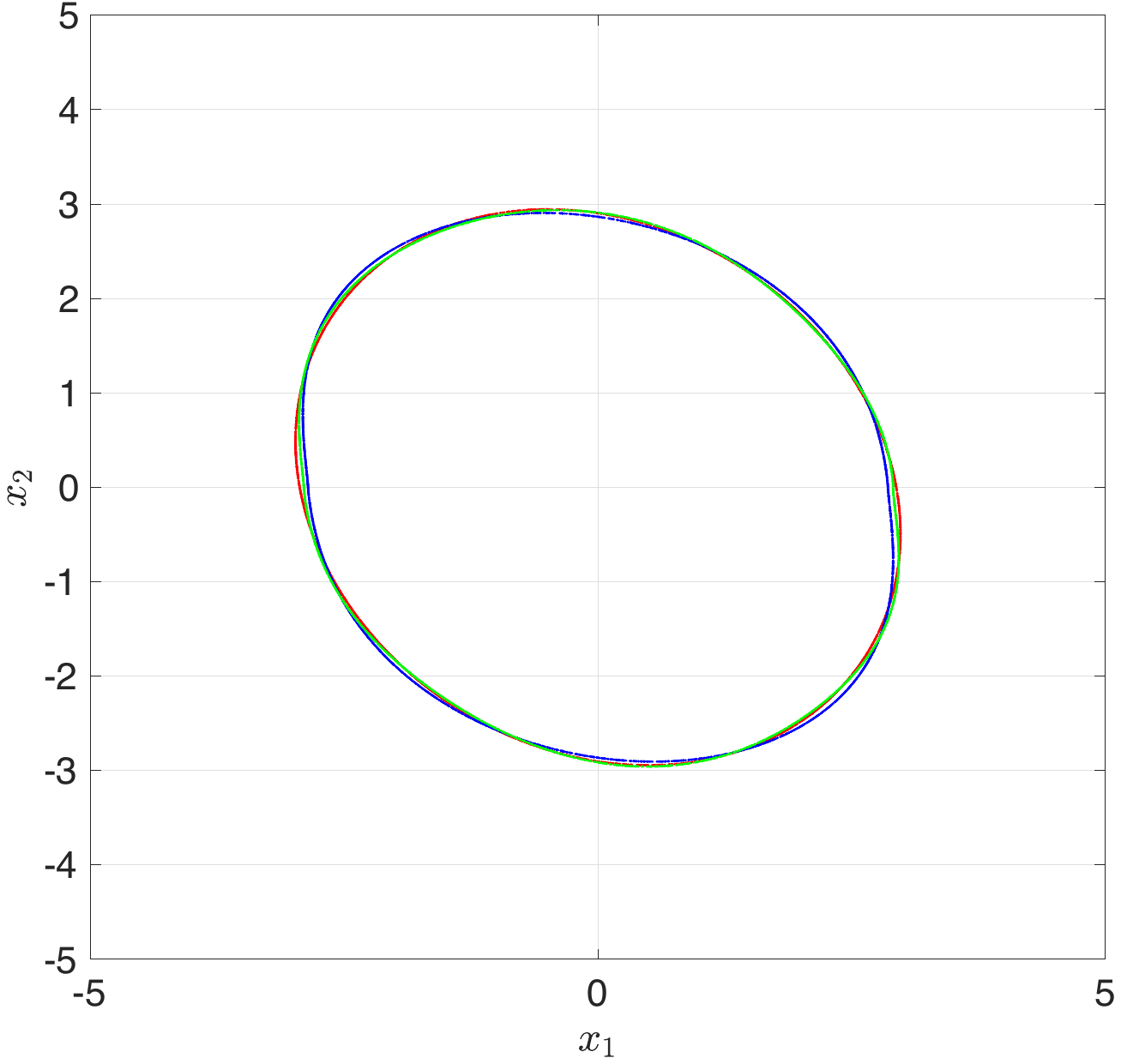}
 	\caption{Level lines $\|x\|_{\dn}=1$ (red), $\|x\|_{\dn^*}=1$ (blue) and $\|x\|_{\dn,\epsilon}=1$ (green)}
 	\label{fig:hnorm_sigmoid}
 \end{figure}
 For $N=20$ we have $\epsilon\approx 5\cdot 10^{-5}$.  The plot for the latter two cases is not presented since the corresponding spheres are visually indistinguishable with the unit sphere. 
  Therefore, the canonical homogeneous norm estimated by the homogeneous ANN can be  efficiently utilized in control engineering for implementation of the implicit homogeneous controller \cite{Polyakov_etal2015:Aut}, \cite{Polyakov2020:Book} in digital devices.  
%

\section{Conclusion}
 The paper introduces a class of ANNs being global approximators  of the so-called generalized homogeneous functions. The homogeneous universal approximation theorem is proven. A dilation symmetry of the homogeneous function allows local approximation results  to be extended globally.   
 
An information about dilation symmetry of a function approximated by a conventional ANN may be unknown. However, the  generalized  homogeneous function theorem proven in the paper shows that this information is, in fact, hidden in ANN provided that the ANN is well trained.  This important information can be extracted (identified) for the ANN by solving a quadratic optimization problem and, next, the conventional ANN can be easily homogenized using the identified dilation. The obtained homogeneous ANN is a global extrapolation of the conventional ANN away from the training set. 
 
 The potential advantages of homogeneous ANN are demonstrated on academic examples from computer science, systems theory and control engineering. Application of homogeneous ANN to practical problems is an interesting problem of future research.  A theoretical study about homogeneous physics-informed ANNs \cite{Lagaris_etal1998:TNN} may also be interesting, since many models of mathematical physics are homogeneous in generalized sense \cite{Polyakov2020:Book}.
 
 \section{Data and Code}
 The code for MATLAB and the trained ANNs for the considered examples are available at
 
 {\footnotesize  \texttt{https://gitlab.inria.fr/polyakov/homogeneous-artificial-neural-networks}}

\bibliographystyle{plain}

\end{document}